\documentclass{article}

\usepackage{microtype}
\usepackage{graphicx}
\usepackage{subfigure}
\usepackage{booktabs} 

\usepackage{hyperref}

\usepackage[accepted]{icml2021}

\icmltitlerunning{Towards Certifying $\ell_\infty$ Robustness using Neural Networks with $\ell_\infty$-dist Neurons}

\usepackage{amsmath,amsfonts,bm}
\newcommand{\rbr}[1]{\left(#1\right)}
\newcommand{\sbr}[1]{\left[#1\right]}
\newcommand{\cbr}[1]{\left\{#1\right\}}

\def\eqref#1{equation~\ref{#1}}

\def\1{\bm{1}}

\def\eps{{\epsilon}}

\def\vb{{\bm{b}}}

\def\ve{{\bm{e}}}

\def\vg{{\bm{g}}}
\def\vh{{\bm{h}}}

\def\vm{{\bm{m}}}

\def\vr{{\bm{r}}}

\def\vw{{\bm{w}}}
\def\vx{{\bm{x}}}

\def\vz{{\bm{z}}}

\DeclareMathAlphabet{\mathsfit}{\encodingdefault}{\sfdefault}{m}{sl}
\SetMathAlphabet{\mathsfit}{bold}{\encodingdefault}{\sfdefault}{bx}{n}

\def\gT{{\mathcal{T}}}

\def\gX{{\mathcal{X}}}
\def\gY{{\mathcal{Y}}}

\def\sF{{\mathbb{F}}}

\def\sK{{\mathbb{K}}}

\def\sS{{\mathbb{S}}}

\newcommand{\R}{\mathbb{R}}

\DeclareMathOperator*{\argmax}{arg\,max}

\DeclareMathOperator{\sign}{sign}

\usepackage{amssymb}
\usepackage{amsthm}
\newtheorem{theorem}{Theorem}[section]
\newtheorem{lemma}[theorem]{Lemma}
\newtheorem{definition}[theorem]{Definition}
\newtheorem{proposition}[theorem]{Proposition}

\newtheorem{remark}[theorem]{Remark}
\newtheorem{fact}[theorem]{Fact}
\usepackage{multirow}

\usepackage{amssymb}
\usepackage{pifont}
\newcommand{\cmark}{\ding{51}}%
\newcommand{\xmark}{\ding{55}}%

\usepackage{enumitem} 

\usepackage{tablefootnote}

\begin{document}

\twocolumn[
\icmltitle{Towards Certifying $\ell_\infty$ Robustness\\ using Neural Networks with $\ell_\infty$-dist Neurons}



\icmlsetsymbol{equal}{*}

\begin{icmlauthorlist}
\icmlauthor{Bohang Zhang}{pku}
\icmlauthor{Tianle Cai}{princeton,tianle}
\icmlauthor{Zhou Lu}{zhoulu}
\icmlauthor{Di He}{ms}
\icmlauthor{Liwei Wang}{pku,pku2}
\end{icmlauthorlist}


\icmlaffiliation{pku}{Key Laboratory of Machine Perception,
MOE, School of EECS, Peking University}
\icmlaffiliation{pku2}{Center for Data Science, Peking University}
\icmlaffiliation{zhoulu}{Department of Computer Science, Princeton University}
\icmlaffiliation{princeton}{Department of Electrical and Computer Engineering, Princeton University}
\icmlaffiliation{tianle}{Zhongguancun Haihua Institute for Frontier Information Technology}
\icmlaffiliation{ms}{Microsoft Research}

\icmlcorrespondingauthor{Liwei Wang}{wanglw@cis.pku.edu.cn}

\icmlkeywords{Adversarial Robustness, Certified Robustness, Lipschitz Network}
\vskip 0.3in
]



\printAffiliationsAndNotice{}  

\begin{abstract}
It is well-known that standard neural networks, even with a high classification accuracy, are vulnerable to small $\ell_\infty$-norm bounded adversarial perturbations. Although many attempts have been made, most previous works either can only provide empirical verification of the defense to a particular attack method, or can only develop a certified guarantee of the model robustness in limited scenarios. In this paper, we seek for a new approach to develop a theoretically principled neural network that inherently resists $\ell_\infty$ perturbations. In particular, we design a novel neuron that uses $\ell_\infty$-distance as its basic operation (which we call $\ell_\infty$-dist neuron), and show that any neural network constructed with $\ell_\infty$-dist neurons (called $\ell_{\infty}$-dist net) is naturally a 1-Lipschitz function with respect to $\ell_\infty$-norm. This directly provides a rigorous guarantee of the certified robustness based on the margin of prediction outputs. We then prove that such networks have enough expressive power to approximate any 1-Lipschitz function with robust generalization guarantee. We further provide a holistic training strategy that can greatly alleviate optimization difficulties. Experimental results show that using $\ell_{\infty}$-dist nets as basic building blocks, we consistently achieve state-of-the-art performance on commonly used datasets: 93.09\% certified accuracy on MNIST ($\epsilon=0.3$), 35.42\% on CIFAR-10 ($\epsilon=8/255$) and 16.31\% on TinyImageNet ($\epsilon=1/255$).
\end{abstract}

\section{Introduction}

Modern neural networks are usually sensitive to small, adversarially chosen perturbations to the inputs \citep{DBLP:journals/corr/SzegedyZSBEGF13, biggio2013evasion}. Given an image $\vx$ that is correctly classified by a neural network, a malicious attacker may find a small adversarial perturbation $\boldsymbol{\delta}$ such that the perturbed image $\vx+\boldsymbol{\delta}$, though visually indistinguishable from the original image, is assigned to a wrong class with high confidence by the network. Such vulnerability creates security concerns in many real-world applications. 

Developing a model that can resist small $\ell_\infty$ perturbations has been extensively studied in the literature. Adversarial training methods \citep{DBLP:journals/corr/SzegedyZSBEGF13,goodfellow2014, madry2017towards, pmlr-v97-zhang19p,ding2020mma} first generate on-the-fly adversarial examples of the inputs, then update model parameters using these perturbed samples together with the original labels. While such approaches can achieve decent empirical robustness, the evaluation is restricted to a particular (class of) attack method, and there are no formal guarantees whether the resulting model is robust against other attacks \citep{DBLP:journals/corr/abs-1802-00420,tramer2020adaptive,tjeng2018evaluating}.

Another line of algorithms train provably robust models for standard networks by maximizing the certified radius provided by robust certification methods, typically using linear relaxation \citep{wong2018provable,pmlr-v80-weng18a,mirman2018differentiable,zhang2018efficient,wang2018mixtrain,NEURIPS2018_f2f44698}, semidefinite relaxation \citep{raghunathan2018certified,pmlr-v115-dvijotham20a}, interval bound relaxation \citep{mirman2018differentiable,gowal2018effectiveness} or their combinations \citep{zhang2020towards}. However, most of these methods are sophisticated to implement and computationally expensive.
Besides these approaches, \citet{pmlr-v97-cohen19c,salman2019provably, zhai2020macer} study the certified guarantee on $\ell_2$ perturbations for Gaussian smoothed classifiers. However, recent works suggest that such methods are hard to extend to the $\ell_\infty$-perturbation scenario if the input dimension is large.

In this work, we propose a new approach by introducing a novel type of neural network that naturally resists local adversarial attacks, and can be easily certified under $\ell_\infty$ perturbation. In particular, we propose a novel neuron called $\ell_\infty$-dist neuron. Unlike the standard neuron design that uses a linear transformation followed by a non-linear activation, the $\ell_\infty$-dist neuron is purely based on computing the $\ell_\infty$-distance between the inputs and the parameters. It is straightforward to see that such a neuron is 1-Lipschitz with respect to $\ell_\infty$-norm, and the neural networks constructed with $\ell_\infty$-dist neurons (called $\ell_{\infty}$-dist nets) enjoy the same property. Based on such a property, we can efficiently obtain the certified robustness for any $\ell_{\infty}$-dist net using the margin of the prediction outputs. 

Theoretically, we investigate the expressive power of $\ell_{\infty}$-dist nets and their robust generalization ability. We first prove a Lipschitz-universal approximation theorem which shows that $\ell_{\infty}$-dist nets can approximate any 1-Lipschitz function (with respect to $\ell_{\infty}$-norm) arbitrarily well. We then give upper bounds of robust test error, which would be small if the $\ell_{\infty}$-dist net learns a large margin classifier on the training data. These results demonstrate the excellent expressivity and generalization ability of the $\ell_{\infty}$-dist net function class.

While $\ell_{\infty}$-dist nets have nice theoretical guarantees, training such a network is still challenging. For example, the gradient of the parameters for $\ell_{\infty}$-norm distance is sparse, which makes the optimization difficult. In addition, we find that commonly used tricks and techniques in conventional network training cannot be taken for granted for this fundamentally different architecture. We address these challenges by proposing a holistic strategy for $\ell_{\infty}$-dist net training. Specifically, we show how to initialize the model parameters, apply proper normalization, design suitable weight decay mechanism, and overcome the sparse gradient problem via smoothed approximated gradients. Using the above methods, training an $\ell_\infty$-dist net is just as easy as training a standard network without any adversarial training, even though the resulting model is already provably robust.

Furthermore, the $\ell_{\infty}$-dist net has wide adaptability by serving as a robust feature extractor and combining itself with conventional networks for practical applications. After building a simple 2-layer perceptron on top of an $\ell_{\infty}$-dist net, we show that the model allows fast training and certification, and consistently achieves state-of-the-art certified robustness on a wide range of classification tasks.
Concretely, we reach \textbf{93.09\%} certified accuracy on MNIST under perturbation $\epsilon=0.3$, \textbf{79.23\%} on FashionMNIST under $\epsilon=0.1$, \textbf{35.42\%} on CIFAR-10 under $\epsilon=8/255$, and \textbf{16.31\%} on TinyImageNet under $\epsilon=1/255$. 
As a comparison, these results outperform the previous best-known results \citep{xu2020automatic}, in which they achieve 33.38\% certified accuracy on CIFAR-10 dataset, and achieve 15.86\% certified accuracy on TinyImageNet using a WideResNet model which is \textit{33 times larger} than the $\ell_\infty$-dist net.

Our contributions are summarized as follows:
\begin{itemize}[topsep=0pt]
\setlength{\itemsep}{0pt}
    \item We propose a novel neural network using $\ell_{\infty}$-dist neurons, called $\ell_{\infty}$-dist net. We show that any $\ell_{\infty}$-dist net is 1-Lipschitz with respect to $\ell_\infty$-norm, which directly guarantees the certified robustness (Section 3).
    \item In the theoretical part, we prove that $\ell_{\infty}$-dist nets can approximate any 1-Lipschitz function with respect to $\ell_\infty$-norm. We also prove that $\ell_{\infty}$-dist nets have a good robust generalization ability (Section 4).
    \item In the algorithmic part, we provide a holistic training strategy for $\ell_{\infty}$-dist nets, including  parameter initialization, normalization, weight decay and smoothed approximated gradients (Section 5).
    \item We show how to combine $\ell_{\infty}$-dist nets with standard networks and obtain robust models more effectively (Section 6). Experimental results show that we can consistently achieve state-of-the-art certified accuracy on MNIST, Fashion-MNIST, CIFAR-10 and TinyImageNet dataset (Section 7).
    \item Finally, we provide all the implementation details and codes at \href{https://github.com/zbh2047/L\_inf-dist-net}{https://github.com/zbh2047/L\_inf-dist-net}.
\end{itemize}

\section{Related Work}
\label{sec_related_work}
\paragraph{Robust Training Approaches.}
Adversarial training is the most successful method against adversarial attacks. By adding adversarial examples to the training set on the fly, adversarial training methods
\citep{DBLP:journals/corr/SzegedyZSBEGF13,goodfellow2014,madry2017towards, huang2015learning,pmlr-v97-zhang19p, wong2020fast,ding2020mma} can significantly improve the robustness of conventional neural networks. However, all the methods above are evaluated according to the empirical robust accuracy against pre-defined adversarial attack algorithms, such as projected gradient decent. These methods cannot guarantee whether the resulting model is also robust against other attacks.

\paragraph{Certified Robustness for Conventional Networks.} 
Many recent works focus on certifying the robustness of learned neural networks under \emph{any} attack.
These approaches are mainly based on bounding the certified radius layer by layer using some convex relaxation methods \citep{wong2018provable,NIPS2018_8060,pmlr-v80-weng18a,mirman2018differentiable,dvijotham2018dual,zhang2018efficient,wang2018mixtrain,NEURIPS2018_f2f44698,xiao2018training,Balunovic2020Adversarial,raghunathan2018certified,pmlr-v115-dvijotham20a}. However, such approaches are usually complicated, computationally expensive and have difficulties in applying to deep and large models. To overcome these drawbacks, \citet{mirman2018differentiable,gowal2018effectiveness} considered interval bound propagation (IBP), a special form of convex relaxation which is much simpler and computationally cheaper. However, the produced bound is loose which results in unstable
training. \citet{zhang2020towards,xu2020automatic} took a further step to combine IBP with linear relaxation to make the bound tighter, which achieves current state-of-the-art performance. Fundamentally different from all these approaches that target to certify conventional networks, we proposed a novel network that provides robustness guarantee by its nature.

\paragraph{Certified Robustness for Smoothed Classifiers.} Randomized smoothing can provide a (probabilistic) certified robustness guarantee for general models. \citet{lecuyer2018certified, DBLP:journals/corr/abs-1809-03113, pmlr-v97-cohen19c, salman2019provably, zhai2020macer,zhang2020black} showed that if a Gaussian random noise is added to the input, a certified guarantee on small $\ell_2$ perturbation can be computed for Gaussian smoothed classifiers. However, \citet{yang2020randomized, blum2020random, kumar2020curse} showed that randomized smoothing cannot achieve nontrivial certified accuracy against larger than $\Omega\rbr{\min\rbr{1, d^{1/p-1/2}}}$ radius for $\ell_{p}$ perturbations, where $d$ is the input dimension. Therefore it cannot provide meaningful results for a relatively large $\ell_{\infty}$ perturbation due to the curse of dimensionality. 

\paragraph{Lipschitz Networks.} Another line of approaches sought to bound the global Lipschitz constant of the neural network. Lipschitz networks can be very useful in certifying adversarial robustness \citep{tsuzuku2018lipschitz}, proving generalization bounds \citep{7934087}, or estimating Wasserstein distance \citep{pmlr-v70-arjovsky17a}. Previous works trained Lipschitz ReLU networks by either directly constraining the spectral norm of each weight matrix to be less than one, or optimizing a loss constructed using the global Lipschitz constant which is upper bounded by these spectral norms  \citep{cisse2017parseval,yoshida2017spectral,gouk2018regularisation,tsuzuku2018lipschitz,qian2018lnonexpansive}. However, as pointed out by \citet{huster2018limitations} and \citet{anil2019sorting}, such Lipschitz networks lack expressivity to some simple Lipschitz functions and the global Lipschitz bound is not tight. Recently, \citet{anil2019sorting} proposed a new Lipschitz network that is a  Lipschitz-universal approximator. \citep{NEURIPS2019_1ce3e6e3} extended their work to convolutional architectures. However, the robustness performances are still not as good as other certification methods, and none of these methods can provide good certified results for $\ell_\infty$ robustness. In this
work, we show the proposed $\ell_\infty$-dist net is also a Lipschitz-universal approximator (under $\ell_\infty$-norm), and the Lipschitz $\ell_\infty$-dist net can substantially outperform other Lipschitz networks in term of certified accuracy.

\paragraph{$\ell_p$-dist Neurons.} We notice that recently \citet{chen2020addernet,xu2020kernel} proposed a new network called AdderNet, which leverages $\ell_1$-norm operation to build the network for sake of efficient inference. \citet{wang2019kervolutional} also considered replacing dot-product neurons by $\ell_1$-dist or $\ell_2$-dist neurons in order to enhance the model's non-linearity and expressivity. Although $\ell_\infty$-dist net looks similar to these networks on the surface, they are designed for fundamentally different problems, and in fact $\ell_1$-dist neurons (or $\ell_2$-dist neurons) can not give any robust guarantee for norm-bounded perturbations (see Appendix \ref{sec_addernet} for more discussions). 

\section{$\ell_\infty$-dist Network and its Robustness Guarantee}
\label{sec_property}
\subsection{Preliminaries} 
\label{sec:bound}
Consider a standard classification task. Suppose we have an underlying data distribution $\mathcal{D}$ over pairs of examples $\vx \in \gX$ and corresponding labels $y\in \gY =\{1,2,\cdots,M\}$ where $M$ is the number of classes. Usually $\mathcal{D}$ is unknown and we can only access a training set $\gT=\{(\vx_1,y_1),\cdots,(\vx_n,y_n)\}$ in which $(\vx_i,y_i)$ is \textit{i.i.d.} drawn from $\mathcal D$. Let $f\in\mathcal{F}$ be the classifier of interest that maps any $\vx\in\gX$ to $\gY$. We call $\vx'=\vx+\boldsymbol{\delta}$ an \textit{adversarial example} of $\vx$ to classifier $f$ if $f$ can correctly classify $\vx$ but assigns a different label to $\vx'$. In real practice, the most commonly used setting is to consider the attack under $\epsilon$-bounded $\ell_{\infty}$-norm constraint, i.e., $\boldsymbol{\delta}$ satisfies $\|\boldsymbol{\delta}\|_{\infty} \leq \eps$, which is also called $\ell_\infty$ perturbations.

Our goal is to learn a model from $\gT$ that can resist attacks at $(\vx,y)$ over $(\vx,y)\sim \mathcal{D}$ for any small $\ell_\infty$ perturbation. It relates to compute the radius of the largest $\ell_{\infty}$ ball centered at $\vx$ in which $f$ does not change its prediction. This radius is called the \textit{robust radius}, which is defined as \citep{zhai2020macer,pmlr-v97-zhang19p}:
\begin{equation}
\label{eqn:def_distance}
R(f;\vx,y)= \left\{
\begin{aligned}
\inf_{f(\vx') \neq f(\vx)} \left\| \vx'-\vx \right\|_{\infty}, & \quad f(\vx)=y \\ 
0 \quad\quad\quad\quad, & \quad f(\vx) \neq y
\end{aligned}
\right .
\end{equation} 
Unfortunately, exactly computing the robust radius  of a classifier induced by a standard deep neural network is very difficult. For example, \citet{katz2017reluplex} showed that calculating such radius for a DNN with ReLU activation is NP-hard. Researchers then seek to derive a tight \textit{lower bound} of $R(f;\vx,y)$ for general $f$. Such lower bound is called  \textit{certified radius} and we denote it as $CR(f;\vx,y)$. It follows that $CR(f;\vx,y) \leq R(f;\vx,y)$ for any $f, \vx, y$. 

\subsection{Networks with $\ell_\infty$-dist Neurons}
In this subsection, we propose a novel neuron called the $\ell_\infty$-dist neuron, which is inherently robust with respect to $\ell_\infty$-norm perturbations. Using these neurons as building blocks, we then show how to obtain robust neural networks dubbed $\ell_\infty$-dist nets.

\begin{figure}
    \centering
    \vspace{0pt}
    \includegraphics[width=0.48\textwidth]{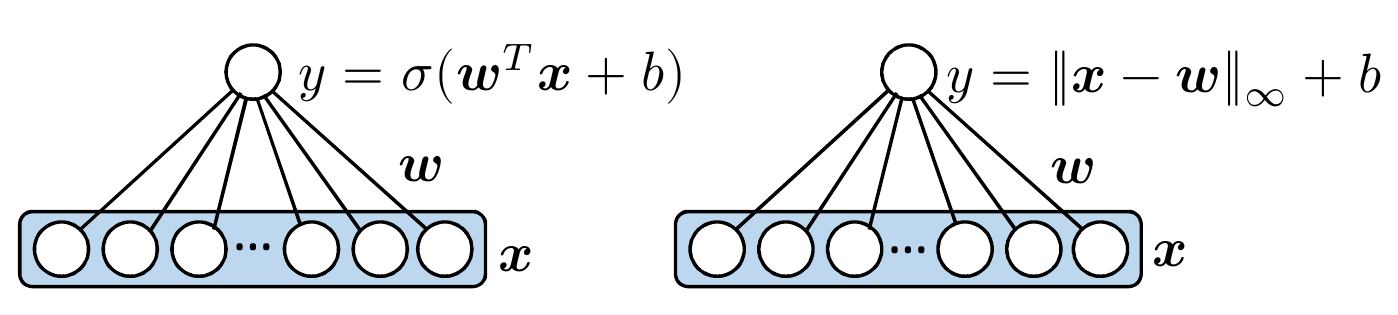}
    \vspace{-20pt}
    \caption{Illustration of the conventional neuron (left) and the $\ell_\infty$-dist neuron (right).}
    \label{fig:neuron}
    \vspace{-10pt}
\end{figure}

Denote $\vx$ as the input vector to a neuron. A standard neuron processes the input by first projecting $\vx$ to a scalar value using a linear transformation, then applying a non-linear activation function $\sigma$ on it, i.e., $\sigma(\vw^{\top}\vx +b)$ with $\vw$ and $b$ as parameters and function $\sigma$ being sigmoid or ReLU activation.
Unlike the previous design paradigm, we introduce a new type of neuron using $\ell_\infty$ distance as the basic operation, called $\ell_\infty$-dist neuron:
\begin{equation}\label{eq:acti}  
u(\vx,\theta)=\|\vx-\vw\|_{\infty}+b,
\end{equation}
where $\theta = \{\vw,b\}$ is the parameter set (see Figure \ref{fig:neuron}  for an illustration). From Eqn.~\ref{eq:acti} we can see that the $\ell_\infty$-dist neuron is non-linear as it calculates the $\ell_\infty$-distance between input $\vx$ and parameter $\vw$ with a bias term $b$. As a result, there is no need to further apply a non-linear activation function.

\begin{remark}
\label{remark_similar}
Conventional neurons use dot-product to represent the similarity between input $\vx$ and weight $\vw$. Likewise, $\ell_\infty$-distance is also a similarity measure. Note that $\ell_\infty$-distance is always non-negative, and a smaller $\ell_\infty$-distance indicates a stronger similarity.
\end{remark}

Without loss of generality, we study the properties of multi-layer perceptron (MLP) networks constructed using $\ell_\infty$-dist neurons. All theoretical results can be easily extended to other neural network architectures, such as convolutional networks. We use $\vx\in \R^{d_{\mathrm{input}}}$ to denote the input vector of an MLP network. An MLP network using $\ell_\infty$-dist neurons can be formally defined as follows.

\begin{definition} ($\ell_\infty$-dist Net)
\label{def_ell_inf_net}
Define an $L$ layer $\ell_\infty$-dist net as follows. Assume the $l$-th hidden layer contains $d_l$ hidden units. The network takes $\vx^{(0)} \triangleq \vx \in \R^{d_{\mathrm{input}}}$ as input, and the $k$-th unit in the $l$-th hidden layer $x^{(l)}_k$ is computed by 
\begin{equation}
  \begin{aligned}
   x^{(l)}_k=u(\vx^{(l-1)}, \theta^{(l,k)})=\|\vx^{(l-1)}-\vw^{(l,k)}\|_{\infty}+b^{(l,k)},\\
 1\le l\le L, 1\le k \le d_l
\end{aligned}  
\end{equation}
where $\vx^{(l)}=(x^{(l)}_1,x^{(l)}_2,\cdots,x^{(l)}_{d_{l}})$ is the output of the $l$-th layer. 
\end{definition}

For classification tasks, the dimension of the final outputs of an $\ell_\infty$-dist net matches the number of categories, i.e., $M$. Based on Remark \ref{remark_similar}, we use the negative of the final layer to be the outputs of the $\ell_\infty$-dist net $\vg$, i.e.  $\vg(\vx)=(-x_1^{(L)}, -x_2^{(L)}, \cdots, -x_M^{(L)})$ and define the predictor $f(\vx) = \argmax_{i \in [M]} g_i(\vx)$. Similar to conventional networks, we can apply any standard loss function on the $\ell_\infty$-dist net, such as the cross-entropy loss or hinge loss.

\subsection{Lipschitz and Robustness Facts about $\ell_{\infty}$-dist Nets}
In this subsection, we will show that the $\ell_\infty$-dist neurons and the neural networks constructed using them have nice theoretical properties in controlling the robustness of the model. We first show that $\ell_{\infty}$-dist nets are 1-Lipschitz with respect to $\ell_\infty$-norm, then derive the certified robustness of the model based on such property.

\begin{definition} (Lipschitz Function)
A function $\vg(\vz): \R^m \rightarrow \R^n$ is called $\lambda$-Lipschitz with respect to $\ell_p$-norm $\|\cdot \|_p$, if for any $\vz_1, \vz_2$, the following holds:
$$\|\vg(\vz_1)-\vg(\vz_2)\|_p\le \lambda \|\vz_1-\vz_2\|_p$$
\end{definition}

\begin{fact}\label{thm:lip}
Any $\ell_\infty$-dist net $\vg(\cdot)$ is 1-Lipschitz with respect to $\ell_\infty$-norm, i.e., for any $\vx_1, \vx_2\in \R^{d_{\mathrm{input}}}$, we have 
$\|\vg(\vx_1)-\vg(\vx_2)\|_{\infty}\le \|\vx_1-\vx_2\|_{\infty}$.
\end{fact}

\begin{proof}
It's easy to check that every basic operation $u(\vx^{(l-1)}, \theta^{(l,k)})$ is 1-Lipschitz, and therefore the mapping from one layer to the next $\vx^{(l)}\to \vx^{(l+1)}$ is 1-Lipschitz. Finally by composition we have for any $\vx_1, \vx_2\in \R^{d_{input}}$, $\|\vg(\vx_1)-\vg(\vx_2)\|_{\infty}\le \|\vx_1-\vx_2\|_{\infty}$.
\end{proof}


\begin{remark}
\label{remark_lip}
Fact \ref{thm:lip} is only true when the basic neuron uses \emph{infinity-norm} distance. For a network constructed using $\ell_p$-dist neurons where $p<\infty$ (e.g. 2-norm), such network will not be 1-Lipschitz (even with respect to $\ell_p$-norm) because the mapping from one layer to the next $\vx^{(l)}\to \vx^{(l+1)}$ is not 1-Lipschitz.
\end{remark}

Since $\vg$ is 1-Lipschitz with respect to $\ell_\infty$-norm, if the perturbation over $\vx$ is rather small, the change of the output can be bounded and the prediction of the perturbed data $\vx'$ will not change as long as $\argmax_{i \in [M]} g_i(\vx) = \argmax_{i \in [M]} g_i(\vx')$, which directly bounds the certified radius.

\begin{fact}\label{thm:mar}
Given model $f(\vx)=\argmax_{i\in[M]}g_i(\vx)$ defined above, and $\vx$ being correctly classified, we define $\mathrm{margin}(\vx;\vg)$ as the difference between the largest and second-largest elements of $\vg(\vx)$. Then for any $\vx'$ satisfying $\|\vx-\vx'\|_{\infty}< \mathrm{margin}(\vx;\vg)/2$, we have that $f(\vx)=f(\vx')$. In other words, 
\begin{equation}\label{eq:cor1}
CR(f,\vx,y)\ge \mathrm{margin}(\vx;\vg)/2
\end{equation}
\end{fact}
\begin{proof}
Since $\vg(\vx)$ is 1-Lipschitz, each element of $\vg(\vx)$ can move at most $\mathrm{margin}(\vx;\vg)/2$ when $\vx$ changes to $\vx'$, therefore the largest element will remain the same.
\end{proof}
Using this bound, we can certify the robustness of an $\ell_{\infty}$-dist net \emph{of any size} under $\ell_{\infty}$-norm perturbations with little computational cost (\emph{only a forward pass}). In contrast, existing certified methods may suffer from either poor scalability (methods based on linear relaxation) or curse of dimensionality (randomized smoothing).

\section{Theoretical Properties of $\ell_{\infty}$-dist Nets}
The expressive power of a model family and its generalization are two central topics in machine learning. Since we have shown that $\ell_{\infty}$-dist nets are 1-Lipschitz with respect to $\ell_\infty$-norm, it's natural to ask whether $\ell_{\infty}$-dist nets can approximate \emph{any} 1-Lipschitz function (with respect to $\ell_\infty$-norm) and whether we can give generalization guarantee on the robust test error based on the Lipschitz property. In this section, we give affirmative answers to both questions. Without loss of generality, we consider \emph{binary classification} problems and assume the output dimension is 1. All the omitted proofs in this section can be found in Appendix \ref{sec_proof}.

\subsection{Lipschitz-Universal Approximation of $\ell_{\infty}$-dist Nets}
It is well-known that the conventional network is a universal approximator, in that it can approximate any continuous function arbitrarily well \citep{cybenko1989approximation}.
Similarly, in this section we will prove a Lipschitz-universal approximation theorem for $\ell_{\infty}$-dist nets, formalized in the following: 

\begin{theorem}\label{thm:inflayer}
For any 1-Lipschitz function $\tilde{g}(\vx)$ (with respect to $\ell_{\infty}$-norm) on a bounded domain $\sK\in \mathbb R^{d_{\text{input}}}$ and any $\epsilon>0$, there exists an $\ell_{\infty}$-dist net $g(\vx)$ with width no more than $d_{input}+2$, such that for all $\vx\in \sK$, we have $\|g(\vx)-\tilde{g}(\vx)\|_{\infty} \le \epsilon$.
\end{theorem}

We briefly present a proof sketch of Theorem $\ref{thm:inflayer}$. The proof has the same structure to \citet{NIPS2017_7203}, who first proved such universal approximation theorem for width-bounded ReLU networks, by constructing a special network that approximates a target function $\tilde{g}(\vx)$ by the sum of indicator functions of grid points, i.e. $\tilde{g}(\vx)\approx\sum_{\vz\in \mathbb S} \tilde{g}(\vz)\mathbf 1_{\{\vx:\|\vx-\vz\|_{\infty}\le \epsilon/2\}}(\vx)$ where $\mathbb S=\{\vz\in \mathbb K:z_i=C_i\epsilon,C_i\in \mathbb Z\}$ . For $\ell_{\infty}$-dist nets, such an approach cannot be directly applied as the summation will break the Lipschitz property. We employ a novel ``max of pyramids'' construction to overcome the issue. The key idea is to approximate the target function using the maximum of many ``pyramid-like'' basic 1-Lipschitz functions, i.e. $\tilde{g}(\vx)\approx g(x):=\max_{\vz\in \mathbb S} (\tilde{g}(\vz)-\|\vx-\vz\|_{\infty})$. To represent such a function, we first show that the $\ell_\infty$-dist neuron can express the following basic functions: $f(\vx)=\|\vx-\vz\|_{\infty}$, $f(\vx)=x_i+b$, $f(\vx)=-x_i+b$ and $f(\vx)=\max(x_i,x_j)$; Then $g(x)$ can be constructed using these functions as building blocks. Finally, we carefully design a computation pattern for a width-bounded $\ell_\infty$-dist net to perform such max-reduction layer by layer.

Theorem \ref{thm:inflayer} implies that an $\ell_{\infty}$-dist net can approximate any 1-Lipschitz function with respect to $\ell_{\infty}$-norm on a compact set, using width barely larger than the input dimension. Combining it with Fact \ref{thm:lip}, we conclude that $\ell_{\infty}$-dist nets are a good class of models to approximate 1-Lipschitz functions.

\subsection{Bounding Robust Test Error of $\ell_{\infty}$-dist Nets}
In this subsection, we give a generalization bound for the \emph{robust test error} of $\ell_{\infty}$-dist nets. Let $(\vx, y)$ be an instance-label couple where $\vx\in \sK$ and $y\in \{1,-1\}$ and denote $\mathcal{D}$ as the distribution of $(\vx, y)$. For a function $g(\vx): \R^{d_{\mathrm{input}}} \rightarrow \R$, we use $\sign(g(\vx))$ as the classifier. The $r$-robust test error $\gamma_r$ of a classifier $g$ is defined as
$$\gamma_r=\mathbb E_{(x,y)\sim\mathcal{ D}}\sbr{\sup_{\|\vx'-\vx\|_{\infty}\le r}\mathbb I_{y g(\vx')\le 0}}.$$
Then $\gamma_r$ can be upper bounded by the margin error on training data and the size of the network, as stated in the following theorem:

\begin{theorem}\label{thm:rob}
Let $\sF$ denote the set of all $g$ represented by an $\ell_{\infty}$-dist net with width $W(W\ge d_{\text{input}})$ and depth $L$. For every $t>0$, with probability at least $1-2e^{-2t^2}$ over the random drawing of $n$ samples, for all $ r>0$ and $g\in \sF$ we have that
\begin{equation}\label{eq:lipgene}
\begin{aligned}
\gamma_r \le \inf_{\delta\in (0,1]} \left[\frac{1}{n}\sum_{i=1}^n \right.\underbrace{\mathbb I_{y_i g(\vx_i)\le \delta+r}}_{\text{large training margin}}+\underbrace{\tilde{O}\rbr{\frac{LW^2}{\delta \sqrt{n}}}}_{\text{network size}}\\
\left.+\rbr{\frac{\log\log_2(\frac{2}{\delta})}{n}}^{\frac{1}{2}}\right]+\frac{t}{\sqrt{n}}.  
\end{aligned}
\end{equation}
\end{theorem}
Theorem \ref{thm:rob} demonstrates that when a large margin classifier is found on training data, and the size of the $\ell_{\infty}$-dist net is not too large, then with high probability, the model can generalize well in terms of \textit{adversarial robustness}.
Note that our bound does not depend on the input dimension, while previous generalization bounds for the general Lipschitz model class (i.e. \citet{luxburg2004distance}) suffer from the curse of dimensionality \cite{neyshabur2017exploring}.





\section{Training $\ell_{\infty}$-dist Nets}
\label{sec_optimize}
In this section we will focus on how to train an $\ell_{\infty}$-dist net successfully. Motivated by the theoretical analysis in previous sections, it
suffices to find a large margin solution on the training data. Therefore we can simply use the standard multi-class hinge loss to obtain a large training margin, similar to \citet{anil2019sorting}. Since the loss is differentiable almost everywhere, any gradient-based optimization method can be used to train $\ell_{\infty}$-dist nets.

However, we empirically find that the optimization is challenging and directly training the network usually \textit{fails} to obtain a good performance. Moreover, conventional wisdom like batch normalization cannot be taken as a grant in the $\ell_{\infty}$-dist net since it will hurt the model's robustness. In the following we will dig into the optimization difficulties and provide a holistic training strategy to overcome them.

\subsection{Normalization}
\label{sec_bn}
One important difference between $\ell_\infty$-dist nets and conventional networks is that for conventional networks, the output of a linear layer is unbiased (zero mean in expectation) under random initialization over parameters \citep{glorot2010understanding,he2015delving}, while the output of an $\ell_\infty$-dist neuron is biased (always being non-negative, assuming no bias term). Indeed, for a weight vector initialized using a standard Gaussian distribution, assume a zero input $\vx^{(0)}=\mathbf 0\in \mathbb R^d$ is fed into an $\ell_\infty$-dist net, the expected output of the first layer can be approximated by $x_j^{(1)}=\mathbb E_{\vw^{(1,j)}}\|\vx^{(0)}-\vw^{(1,j)}\|_\infty\approx \sqrt{2\log d}$. The output vector $\vx^{(1)}$ is then fed into subsequent layers, making the outputs in upper layers linearly increase. 

Normalization is a useful way to control the scale of the layer's outputs to a standard range. Batch Normalization \citep{ioffe2015batch}, which shifts and scales feature values in each layer, is shown to be one of the most important components in training deep neural networks. However, if we directly apply batch normalization in $\ell_{\infty}$-dist nets, the Lipschitz constant will change due to the scaling operation, and the robustness of the model cannot be guaranteed. 

Fortunately, we find using the shift operation alone already helps the optimization. Therefore we apply the shift operation in all intermediate layers after calculating the $\ell_{\infty}$ distance. As a result, we remove the bias terms in the corresponding $\ell_{\infty}$-dist neurons as they are redundant. We do not use normalization in the final layer. Similar to BatchNorm, we use the running mean during inference, which serves as additional bias terms in $\ell_{\infty}$-dist neurons and does not affect the Lipschitz constant of the model. We do not use affine transformation, which is typically used in BatchNorm.

\subsection{Smoothed Approximated Gradients}
\label{sec_app_grad}
We find that training an $\ell_{\infty}$-dist net from scratch is usually inefficient, and the optimization can easily be stuck at some bad solution. One important reason is that the gradients of the $\ell_\infty$-dist operation (i.e., $\nabla_{\vw}\|\vz-\vw\|_{\infty}$ and $\nabla_{\vz}\|\vz-\vw\|_{\infty}$) are very sparse which typically contain only one non-zero element. In practice, we observe that there are less than 1\%  parameters updated in an epoch if we directly train the $\ell_{\infty}$-dist net using SGD/Adam from random initialization. 

To improve the optimization, we relax the $\ell_{\infty}$-dist neuron by using the $\ell_{p}$-dist neuron for the whole network to get an approximate and non-sparse gradient of the model parameters. During training, we set $p$ to be a small value in the beginning and increase it in each iteration until it approaches infinity. For the last few epochs, we set $p$ to infinity and train the model to the end. Empirically, using smoothed approximated gradients significantly boosts the performance, as will be shown in Section \ref{sec_ablation}.

\subsection{Parameter Initialization}

We find that there are still optimization difficulties in training \emph{deep} models using normalization and the smoothed gradient method. In particular, we find that a deeper model performs worse than its shallow counterpart in term of \emph{training accuracy} (see Appendix \ref{sec_initialization}), a phenomenon similar to \citet{he2016deep}. To fix the problem, \citet{he2016deep} proposed ResNet architecture by modifying the network using identity mapping as skip connections. According to Proposition \ref{prop:3p} (see Appendix \ref{sec_proof}), an $\ell_\infty$-dist layer can also perform identity mapping by assigning proper weights and biases at initialization, and a deeper $\ell_\infty$-dist net then can act as a shallow one by these identity mappings.

Given such findings, we can directly construct identity mappings at initialization.
Concretely, for an $\ell_\infty$-dist layer with the same input-output dimension, we first initialize the weights randomly from a standard Gaussian distribution as common, then modify the diagonal elements (i.e. $w_j^{(l,j)}$ in Definition \ref{def_ell_inf_net}) to be a large negative number $C_0$. Throughout all experiments, we set $C_0=-10$. We do not need the bias in $\ell_\infty$-dist neurons after applying mean shift normalization, and the running mean automatically makes an identity mapping.

\subsection{Weight Decay}
Weight decay is a commonly used trick in training deep neural networks. It is equivalent to adding an $\ell_2$ regularization term to the loss function. However, we empirically found that using weight decay in the $\ell_\infty$-dist net gives inferior performance (see Section \ref{sec_ablation}). The problem might be the incompatibility of weight decay ($\ell_2$ regularization) with $\ell_\infty$-norm used in  $\ell_\infty$-dist nets, as we will explain below.

Conventional networks use dot-product as the basic operation, therefore an $\ell_2$-norm constraint on the weight vectors directly controls the scale of the output magnitude. However, it is straightforward to see that the $\ell_2$-norm of the weight vector does not correspond to the output scale of the $\ell_\infty$ distance operation. A more reasonable choice is to use $\ell_\infty$-norm regularization instead of $\ell_2$-norm. In fact, we have $\|\vx-\vw\|_\infty\le \|\vx\|_\infty+\|\vw\|_\infty$, analogous to $\langle \vx,\vw\rangle \le \|\vx\|_2\|\vw\|_2$.

For general $\ell_p$-dist neurons during training, we can use $\ell_p$-norm regularization analogously. By taking derivative with respect to the weight $\vw$, we derive the corresponding weight decay formula:
\vspace{-8pt}
\begin{equation}
\label{eq:weight_decay}
    \Delta_{w_i}=-\lambda \nabla_{w_i} \|\vw\|_p^2=-\lambda \left(\frac {|w_i|}{\|\vw\|_p}\right)^{p-2} w_i
\end{equation}
where $\lambda$ is the weight decay coefficient. Note that Eqn.~\ref{eq:weight_decay} reduces to commonly used weight decay if $p=2$. When $p\rightarrow \infty$, the weight decay tends to take effects only on the element $w_i$ with the largest absolute value. 

\section{Certified Robustness by Using $\ell_\infty$-dist Nets as Robust Feature Extractors}
\label{sec_combine}
We have shown that the $\ell_\infty$-dist net is globally Lipschitz in the sense that the function is 1-Lipschitz everywhere over the input space. This constraint is pretty strong when we only require function's Lipschitzness on a specific manifold, such as real image data manifold. To make the model more flexible and fit the practical tasks better, we can build a lightweight conventional network on top of an $\ell_\infty$-dist net. The $\ell_\infty$-dist net $\vg$ will serve as a robust feature extractor, and the lightweight network $\vh$ (a shallow MLP in our experiments) will focus on task-specific goals, such as classification. We denote the composite network as $\vh\circ\vg$.

We first show how to certify the robustness for the composite network $\vh\circ\vg$. Given any input $\vx$, consider a perturbation set $\mathcal B_{\infty}^{\epsilon}(\vx)=\{\vx':\|\vx'-\vx\|_\infty\le \epsilon\}$, where $\epsilon$ is the pre-defined perturbation level. If $(\vh\circ\vg)(\vx')$ predicts the correct label $y$ for all $\vx'\in \mathcal B_{\infty}^{\epsilon}(\vx)$, we can guarantee robustness of the network $\vh\circ\vg$ for input $\vx$. Since $\vg$ is 1-Lipschitz with respect to $\ell_\infty$-norm, we have $\vg(\vx')\in \mathcal B_{\infty}^{\epsilon}(\vg(\vx))$ by definition. Based on this property, to guarantee robustness of the network $\vh\circ\vg$ for input $\vx$, it suffices to check whether for all $\vz'\in \mathcal B_{\infty}^{\epsilon}(\vg(\vx))$, $\vh(\vz')$ predicts the correct label $y$. This is equivalent to certifying the robustness for a conventional network $\vh$ given input $\vz=\vg(\vx)$, which can be calculated using any previous certification method such as convex relaxation.
We describe one of the simplest convex relaxation method named IBP \citep{gowal2018effectiveness} in Appendix \ref{sec_ibp}, which is used in our experiments. Other more advanced approaches such as CROWN-IBP \citep{zhang2020towards,xu2020automatic} can also be considered.

After obtaining the bound, we can set it as the training objective function to train the neural network parameters, similar to \citet{gowal2018effectiveness,zhang2020towards}. All calculations are differentiable, and gradient-based optimization methods can be applied. Note that since $\vh$ and $\vg$ have entirely different architectures, we apply the training strategy introduced in Section \ref{sec_optimize} to the $\ell_\infty$-dist net $\vg$ only. More details will be presented in Section \ref{sec_experiment_setting}.

\section{Experiments \& Results}
\label{sec_experiment}
In this section, we conduct extensive experiments for the proposed network. We train our models on four popular benchmark datasets: MNIST, Fashion-MNIST, CIFAR-10 and TinyImagenet.

\subsection{Experimental Setting}
\label{sec_experiment_setting}

\paragraph{Model details.} We mainly study two types of models. The first type is denoted as $\ell_\infty$-dist Net, i.e., a network consists of $\ell_\infty$-dist neurons only. The second type is a composition of $\ell_\infty$-dist Net and a shallow MLP (denoted as $\ell_\infty$-dist Net+MLP). That is, using $\ell_\infty$-dist Net as a robust feature extractor as described in Section \ref{sec_combine}. We use a 5-layer $\ell_{\infty}$-dist Net for MNIST and Fashion-MNIST, and a 6-layer $\ell_{\infty}$-dist Net for CIFAR-10 and TinyImageNet.
Each hidden layer has 5120 neurons, and the top layer has 10 neurons (or 200 neurons for TinyImageNet) for classification. Normalization is applied between each intermediate layer. For $\ell_\infty$-dist Net+MLP, we remove the top layer and add a 2-layer fully connected conventional network on top of it. The hidden layer has 512 neurons with tanh activation. See Table \ref{tbl_architecture} for a complete demonstration of the models on each dataset.

\paragraph{Training configurations.} In all experiments, we train $\ell_\infty$-dist Net and $\ell_\infty$-dist Net+MLP using Adam optimizer with hyper-parameters $\beta_1=0.9$, $\beta_2=0.99$ and $\epsilon=10^{-10}$. The batch size is set to 512. For data augmentation, we use random crop (padding=1) for MNIST and Fashion-MNIST, and use random crop (padding=4) and random horizontal flip for CIFAR-10, following the common practice. For TinyImageNet dataset, we use random horizontal flip and crop each image to $56\times 56$ pixels for training, and use a center crop for testing, which is the same as \citet{xu2020automatic}. As for the loss function, we use multi-class hinge loss for $\ell_\infty$-dist Net and the IBP loss \citep{gowal2018effectiveness} for $\ell_\infty$-dist Net+MLP. The training procedure is as follows. First, we relax the $\ell_{\infty}$-dist net to $\ell_{p}$-dist net by setting $p=8$ and train the network for $e_1$ epochs. Then we gradually increase $p$ from 8 to 1000 exponentially in the next $e_2$ epochs. Finally, we set $p=\infty$ and train the last $e_3$ epochs. Here $e_1,e_2$ and $e_3$ are hyper-parameters varying from the dataset. We use $lr=0.02$ in the first $e_1$ epochs and decease the learning rate using cosine annealing for the next $e_2+e_3$ epochs. We use $\ell_p$-norm weight decay for $\ell_{\infty}$-dist nets and $\ell_2$-norm weight decay for the MLP with coefficient $\lambda=0.005$. All these explicitly specified hyper-parameters are kept fixed across different architectures and datasets. For $\ell_\infty$-dist Net+MLP training, we use the same linear warmup strategy for hyper-parameter $\epsilon_{\text{train}}$ in \citet{gowal2018effectiveness,zhang2020towards}. See Appendix \ref{sec_experiment_details} (Table \ref{tbl_hyperparameter}) for details of training configuration and hyper-parameters.


\begin{table*}[ht]
\vspace{-5pt}
\caption{Comparison of our results with existing methods\footnotemark[1].}
\small
\vspace{2pt}
\label{tbl_results}
\begin{center}
\begin{tabular}{c|cr|rrr}
\hline
Dataset                                                                                     & Method     & FLOPs & Test          & Robust        & Certified                          \\ \hline
\multirow{6}{*}{\begin{tabular}[c]{@{}c@{}}MNIST\\ ($\epsilon=0.3$)\end{tabular}}           & Group Sort~\citep{anil2019sorting}   &2.9M                      & 97.0 & 34.0 & 2.0  \\

                                                                                            & COLT~\citep{Balunovic2020Adversarial}   &4.9M                      & 97.3                              & -                                 & 85.7                             \\
                                                                                            & IBP ~\citep{gowal2018effectiveness}    &114M                      & 97.88                             & 93.22                             & 91.79                             \\
                                                                                            & CROWN-IBP~\citep{zhang2020towards}  &114M                      & 98.18                             & 93.95                             & 92.98                            \\

                                                                                            & $\ell_\infty$-dist Net     &82.7M                     & 98.54                             & 94.71                            & 92.64                             \\
                                                                                            & $\ell_\infty$-dist Net+MLP   &85.3M                     & \textbf{98.56}                        & \textbf{95.28}                             & \textbf{93.09}                            \\ \hline
\multirow{5}{*}{\begin{tabular}[c]{@{}c@{}}Fashion\\ MNIST\\ ($\epsilon=0.1$)\end{tabular}} & CAP~\citep{wong2018provable}      &0.41M                     & 78.27                             & 68.37                             & 65.47                            \\
                                                                                            & IBP~\citep{gowal2018effectiveness}      &114M                      & 84.12                             & 80.58                             & 77.67                           \\
                                                                                            & CROWN-IBP~\citep{zhang2020towards}    &114M                      & 84.31                             & 80.22                             & 78.01                           \\
                                                                                            & $\ell_\infty$-dist Net     &82.7M                     & \textbf{87.91}                             & 79.64                             & 77.48                             \\
                                                                                            & $\ell_\infty$-dist Net+MLP  &85.3M                     & \textbf{87.91}                            & \textbf{80.89}                             & \textbf{79.23}                           \\ \hline
\multirow{7}{*}{\begin{tabular}[c]{@{}c@{}}CIFAR-10\\ ($\epsilon=8/255$)\end{tabular}}      & PVT~\citep{dvijotham2018training}         &2.4M                      & 48.64                             & 32.72                             & 26.67                            \\
                                                                                            & DiffAI~\citep{mirman2019provable}  &96.3M                     & 40.2                              & -                                 & 23.2                             \\
                                                                                            & COLT~\citep{Balunovic2020Adversarial}   &6.9M                      & 51.7                              & -                                 & 27.5                             \\
                                                                                            & IBP~\citep{gowal2018effectiveness}   &151M                      & 50.99                             & 31.27                             & 29.19                            \\
                                                                                            & CROWN-IBP~\citep{zhang2020towards}   &151M                      & 45.98                             & 34.58                             & 33.06                            \\
                                                                                                                                                                                        & CROWN-IBP (loss fusion)~\citep{xu2020automatic}     &151M                      & {46.29}                            & {35.69}                            & 33.38                          \\
                                                                                            & $\ell_\infty$-dist Net     &121M                      & \textbf{56.80}                            & \textbf{37.46}                            & 33.30                          \\
                                                                                            & $\ell_\infty$-dist Net+MLP  &123M                      & 50.80                             & 37.06                             & \textbf{35.42}                            \\ \hline
\end{tabular}
\end{center}
\vspace{-12pt}
\end{table*}

\paragraph{Evaluation.} Following common practice, we test the robustness of the trained models under $\ell_\infty$-perturbation $\epsilon=0.3$ on MNIST, $\epsilon=0.1$ on Fashion-MNIST, $\epsilon=8/255$ on CIFAR-10 and $\epsilon=1/255$ on TinyImageNet. We use two evaluation metrics to measure the robustness of the model. We first evaluate the robust test accuracy under the Projected Gradient Descent (PGD) attack \citep{madry2017towards}. Following standard practice, we set the number of steps of the PGD attack to be 20. We also calculate the certified radius for each sample, and check the percentage of test samples that can be certified to be robust within the chosen radius. Note that the second metric is always lower than the first.

\paragraph{Baselines.}
We compare our proposed models with state-of-the-art methods for each dataset, including relaxation methods: CAP~\citep{wong2018provable}, PVT~\citep{dvijotham2018training}, DiffAI~\citep{mirman2019provable}, IBP~\citep{gowal2018effectiveness}, CROWN-IBP~\citep{zhang2020towards}, CROWN-IBP with loss fusion~\citep{xu2020automatic}, COLT~\citep{Balunovic2020Adversarial}, and Lipschitz networks: GroupSort~\citep{anil2019sorting}. We do not compare with randomized smoothing methods~\citep{pmlr-v97-cohen19c,salman2019provably,zhai2020macer} since it cannot obtain good certification under s relative large $\ell_\infty$-norm perturbation as discussed in Section \ref{sec_related_work}. We report the performances picked from the original papers if not specified otherwise.

\subsection{Experimental Results}
We list our results in Table~\ref{tbl_results} and Table~\ref{tbl_results_imagenet}. We use ``Standard'', ``Robust'' and ``Certified'' as abbreviations of standard (clean) test accuracy, robust test accuracy under PGD attack and certified robust test accuracy. All the numbers are reported in percentage. We use ``FLOPs'' to denote the number of basic floating-point operations needed (i.e., multiplication-add in conventional networks or subtraction in $\ell_\infty$-dist nets) in forward propagation. Note that FLOPs in linear relaxation methods are typically small because, in the original papers, they are implemented on small networks due to high costs.

\begin{table*}[ht]
\vspace{-5pt}
\caption{Comparison of our results with \citet{xu2020automatic} on TinyImageNet dataset ($\epsilon=1/255$).}
\vspace{2pt}
\small
\label{tbl_results_imagenet}
\begin{center}
\begin{tabular}{c|cr|rrr}
\hline
Method                                                                             & Model                      & \multicolumn{1}{c|}{FLOPs} & \multicolumn{1}{c}{Test} & \multicolumn{1}{c}{Robust} & \multicolumn{1}{c}{Certified} \\ \hline
\multirow{4}{*}{\begin{tabular}[c]{@{}c@{}}CROWN-IBP\\ (loss fusion)\\\citep{xu2020automatic}\end{tabular}} & CNN7+BN                    & 458M                      & 21.58                    & 19.04                      & 12.69                         \\
                                                                                   & ResNeXt                    & 64M                       & 21.42                    & 20.20                      & 13.05                         \\
                                                                                   & DenseNet                   & 575M                      & 22.04                    & 19.48                      & 14.56                         \\
                                                                                   & WideResNet                 & 5.22G                     & 27.86                    & 20.52                      & 15.86                         \\ \hline
$\ell_\infty$-dist net                                                             & $\ell_\infty$-dist Net+MLP & 156M                      & 21.82                    & 18.09                      & \textbf{16.31}                         \\ \hline
\end{tabular}
\end{center}
\end{table*}

\begin{table*}[ht]
\vspace{-5pt}
\caption{Ablation studies for Section \ref{sec_optimize} on CIFAR-10 dataset.}
\vspace{2pt}
\small
\label{tbl_ablation}
\begin{center}
\begin{tabular}{c|ccc|rrr|rrr}
\hline
\multicolumn{1}{l|}{} & \multicolumn{1}{c}{\multirow{2}{*}{\begin{tabular}[c]{@{}c@{}}Smooth\\ gradient\end{tabular}}} & \multicolumn{1}{c}{\multirow{2}{*}{\begin{tabular}[c]{@{}c@{}}Identity\\ init\end{tabular}}} & \multicolumn{1}{c|}{\multirow{2}{*}{\begin{tabular}[c]{@{}c@{}}Weight\\ decay\end{tabular}}} & \multicolumn{3}{c|}{$\ell_\infty$-dist Net}                                             & \multicolumn{3}{c}{$\ell_\infty$-dist Net+MLP}                                       \\
\multicolumn{1}{l|}{} & \multicolumn{1}{c}{}                                                                           & \multicolumn{1}{c}{}                                                                         & \multicolumn{1}{c|}{}                                                                        & \multicolumn{1}{c}{Test} & \multicolumn{1}{c}{Robust} & \multicolumn{1}{c|}{Certified} & \multicolumn{1}{c}{Test} & \multicolumn{1}{c}{Robust} & \multicolumn{1}{c}{Certified} \\ \hline
A & \xmark & \xmark & \xmark & 28.21                    & 8.42                       & 7.02                           & 37.68 &28.72 &27.76                      \\
B & \cmark & \xmark & \xmark & 55.63                    & 35.28                        & 32.56                          &  37.61 & 30.99 & 29.71                      \\

C & \cmark & \cmark & \xmark & 56.15                     & 35.96                      & 32.71                           & 48.97                    & 36.21                      & 35.02                         \\
D & \cmark & \cmark & $\ell_2$-norm & 53.64                    & 32.12                      & 29.01                          & 45.34                    & 33.48                      & 32.49                         \\
E & \cmark & \cmark & \cmark & 56.80                    & 37.46                      & 33.30                          & 50.80                    & 37.06                      & 35.42                         \\ \hline
\end{tabular}
\vspace{-10pt}
\end{center}
\end{table*}

\paragraph{General Performance of $\ell_\infty$-dist Net.} From Table \ref{tbl_results} we can see that using $\ell_\infty$-dist Net alone already achieves decent certified accuracy on all datasets. Notably, $\ell_\infty$-dist Net reaches the start-of-the-art certified accuracy on CIFAR-10 dataset while achieving a significantly higher standard accuracy than all previous methods. Note that we just use a standard loss function to train the $\ell_\infty$-dist Net \textit{without any adversarial training}.

\paragraph{General Performance of $\ell_\infty$-dist Net+MLP.} For all these datasets, $\ell_\infty$-dist Net+MLP achieves better certified accuracy than $\ell_\infty$-dist Net, establishing new state-of-the-art results. For example, on the MNIST dataset, the model can reach 93.09\% certified accuracy and 98.56\% standard accuracy; on CIFAR-10 dataset, the model can reach 35.42\% certified accuracy, which is 6.23\% higher than IBP and 2.04\% higher than the previous best result; on the TinyImageNet dataset (see Table \ref{tbl_results_imagenet}), the simple $\ell_\infty$-dist Net+MLP model (156M FLOPs computational cost) already beats the previous best result in \citet{xu2020automatic} although their model is 33 times larger than ours (5.22G FLOPs). Furthermore, the clean test accuracy of $\ell_\infty$-dist Net+MLP is also better than CROWN-IBP on MNIST, Fashion-MNIST and CIFAR-10 dataset.

\paragraph{Efficiency.} Both the training and the certification of $\ell_\infty$-dist Net are very fast. As stated in previous sections, the computational cost per iteration for training $\ell_\infty$-dist Net is roughly the same as training a conventional network of the same size, and the certification process only requires a forward pass. In Table \ref{tab:speed}, we quantitatively compare the per-epoch training speed of our method with previous methods such as IBP or CROWN-IBP on CIFAR-10 dataset. All these experiments are run on a single NVIDIA-RTX 3090 GPU. As we can see, for both $\ell_\infty$-dist Net and $\ell_\infty$-dist Net+MLP, the per-epoch training time is less than 20 seconds, which is significantly faster than CROWN-IBP and is comparable to IBP.

\begin{table}[ht]
    \vspace{-5pt}
    \caption{Comparison of per-epoch training speed for different methods on CIFAR-10 dataset.}
    \vspace{2pt}
    \small
    \label{tab:speed}
    \centering
    \begin{tabular}{cc}
    \hline
        Method & Per-epoch Time (seconds)\\
    \hline
        IBP & 17.4 \\
        CROWN-IBP & 112.4 \\
        CROWN-IBP (loss fusion) & 43.3\\
        $\ell_\infty$-dist Net & 19.7\\ 
        $\ell_\infty$-dist Net+MLP & 19.7\\ 
    \hline
    \end{tabular}
    \vspace{-10pt}
\end{table}

\footnotetext[1]{For GroupSort, the results are obtained from \citet{anil2019sorting}, Fig. 8,9. For COLT, the certified result uses MILP (mixed integer linear programming) solver which is much slower than other methods. For IBP, the results are obtained from \citet{zhang2020towards}. For PVT, the certified accuracy are under perturbation $\epsilon=0.03$ which is smaller than 8/255.}

\paragraph{Discussions with GroupSort Network.} We finally make a special comparison with GroupSort, because both GroupSort network and our proposed $\ell_\infty$-dist Net design 1-Lipschitz networks explicitly and can be directly trained using a standard loss function. However, in the GroupSort network, all weight matrices $\mathbf W$ are constrained to have bounded $\ell_\infty$-norm, i.e., $\|\mathbf W\|_\infty\le 1$, leading to a time-consuming projection operation (see Appendix C in \citet{anil2019sorting}). This operation brings optimization difficulty \citep{cohen2019universal} and further limit the scalability of the network structure. We hypothesis that this is the major reason why $\ell_\infty$-dist Net substantially outperforms GroupSort on the MNIST dataset, as shown in Table \ref{tbl_results}.


\subsection{Ablation Studies}
\label{sec_ablation}
In this section, we conduct ablation experiments to see the effect of smoothed approximated gradients, parameter initialization using identity map construction, and $\ell_p$-norm weight decay. The results are shown in Table \ref{tbl_ablation}. We use the model described in Section \ref{sec_experiment_setting} on CIFAR-10 dataset with hyper-parameters provided in Table \ref{tbl_hyperparameter}. From Table \ref{tbl_ablation} we can clearly see that:
\begin{itemize}[topsep=0pt]
\setlength{\itemsep}{0pt}
    \item Smoothed approximated gradient technique is crucial to train a good model for $\ell_\infty$-dist Net. After applying smoothed approximated gradient only, we can already achieve 32.56\% certified accuracy;
    \item Both smoothed approximated gradient and identity-map initialization are crucial to train a good model for $\ell_\infty$-dist Net+MLP. Combining the two techniques results in 35.02\% certified accuracy;
    \item $\ell_p$-norm weight decay can further boost the results, although the effect may be marginal (0.59\% and 0.4\% improvement of certified accuracy for the two models).
    \item Conventional $\ell_2$-norm weight decay does harm to the performance of $\ell_\infty$-dist nets.
\end{itemize}
In summary, these optimization strategies in Section \ref{sec_optimize} all contribute to the final performance of the model.

\section{Conclusion}
In this paper, we design a novel neuron that uses $\ell_\infty$ distance as its basic operation. We show that the neural network constructed with $\ell_\infty$-dist neuron is naturally a 1-Lipschitz function with respect to $\ell_\infty$ norm. This directly provides a theoretical guarantee of the certified robustness based on the margin of the prediction outputs. We further formally analyze the expressive power and the robust generalization ability of the network, and provide a holistic training strategy to handle optimization difficulties encountered in training $\ell_\infty$-dist nets. Experiments show promising results on MNIST,  Fashion-MNIST, CIFAR-10 and TinyImageNet datasets. As this structure is entirely new, plenty of aspects are needed to investigate, such as how to further handle the optimization difficulties for this network. For future work we will study these aspects and extend our model to more challenging tasks like ImageNet.

\section*{Acknowledgements}
This work was supported by National Key R\&D Program of China (2018YFB1402600), Key-Area Research and Development Program of Guangdong Province (No. 2019B121204008), BJNSF (L172037) and Beijing Academy of Artificial Intelligence. Project 2020BD006 supported by PKU-Baidu Fund.

\bibliography{ref}
\bibliographystyle{icml2021}

\newpage
\appendix

\section{Proof of Theorems}
\label{sec_proof}
\subsection{Proof of Theorem \ref{thm:inflayer}}
To prove Theorem \ref{thm:inflayer}, we need the following key lemma:
\begin{lemma}\label{lem:key}
For any 1-Lipschitz function $f(\vx)$ (with respect to $\ell_{\infty}$ norm) on a bounded domain $\sK\in \mathbb R^n$, and any $\epsilon>0$, there exists a finite number of functions $f_i(\vx)$ such that for all $\vx\in \sK$
$$\max_i f_i(\vx)\le f(\vx)\le \max_i f_i(\vx)+\epsilon,$$
where each $f_i(\vx)$ has the following form 
$$f_i(\vx)=\min_{1\le j\le n}\{x_j-w^{(i)}_j,w^{(i)}_j-x_j\} +b_i.$$
\end{lemma}
\begin{proof}[Proof of Lemma \ref{lem:key}]
Without loss of generality we may assume $\sK\in [0,1]^n$. Consider the set $\sS$ consisting of all points $\frac{\epsilon}{2}(N_1,...,N_n)$ where $N_j$ are integers, we can write $\sS\cap \sK=\{\vw^{(1)},...,\vw^{(m)} \}$ since it's a finite set. $\forall \vw^{(i)} \in \sS\cap \sK$, we define the corresponding $f_i(\vx)$ as follows
\begin{equation}\label{eq:lem1}
\begin{aligned}
f_i(\vx)&=\min_{1\le j \le n} \{x_j-w^{(i)}_j,w^{(i)}_j-x_j\} +f(\vw^{(i)})\\
&=-\|\vx-\vw^{(i)}\|_\infty+f(\vw^{(i)})
\end{aligned}
\end{equation}
On the one hand, we have $f(\vx)-f_i(\vx)=f(\vx)-f(\vw^{(i)})+\|\vx-\vw^{(i)}\|_\infty\ge -\|\vx-\vw^{(i)}\|_\infty+\|\vx-\vw^{(i)}\|_\infty=0$, therefore $\forall \vx \in \sK$ we have $f(\vx)\ge f_i(\vx)$, namely $\max_i f_i(\vx)\le f(\vx)$ holds.

On the other hand, $\forall \vx \in \sK$ there exists its 'neighbour' $\vw^{(j)}\in \sS\cap \sK$ such that $\|\vx-\vw^{(j)}\|_{\infty}\le \frac{\epsilon}{2}$, therefore by using the Lipschitz properties of both $f(\vx)$ and $f_j(\vx)$, we have that
\begin{equation}\label{eq:lem2}
\begin{aligned}
f(\vx)&\le f(\vw^{(j)})+\frac{\epsilon}{2}= f_j(\vw^{(j)})+\frac{\epsilon}{2}\\
&\le f_j(\vx)+\epsilon\le \max_i f_i(\vx)+\epsilon
\end{aligned}
\end{equation}
Combining the two inequalities concludes our proof.
\end{proof}
Lemma \ref{lem:key} ``decomposes'' any target 1-Lipschitz function into simple ``base functions'', which will serve as building blocks in proving the main theorem. We are ready to prove Theorem \ref{thm:inflayer}:

\begin{proof}[Proof of Theorem \ref{thm:inflayer}]
By Lemma \ref{lem:key}, there exists a finite number of functions $f_i(\vx)$ ($i=1,...,m$) such that $\forall \vx\in \sK$
\begin{equation}\label{eq:key2}
\max_i f_i(\vx)\le \tilde{g}(\vx)\le \max_i f_i(\vx)+\epsilon
\end{equation}

where each $f_i(\vx)$ has the form 
\begin{equation}\label{eq:formthm2}
f_i(\vx)=\min_{1\le j\le n}\{x_j-w^{(i)}_j,w^{(i)}_j-x_j\} +\tilde{g}(\vw^{(i)})
\end{equation}

The high-level idea of the proof is very simple: among width $d_{\mathrm{input}}+2$, we allocate $d_{\mathrm{input}}$ neurons each layer to keep the information of $\vx$, one to calculate each $f_i(\vx)$ one after another and the last neuron calculating the maximum of $f_i(\vx)$ accumulated.

To simplify the proof, we would first introduce three general basic maps which can be realized at a single unit, then illustrate how to represent $\max_i f_i(\vx)$ by combing these basic maps.

Let's assume for now that any input to any unit in the whole network has its $\ell_{\infty}$-norm upper bounded by a large constant $C$, we will come back later to determine this value and prove its validity.

\begin{proposition}\label{prop:3p}
$\forall j,k,p$ and constant $w, c$, the following base functions are realizable at the $k$th unit in the $l$th hidden layer:

\textit{1, the projection map:}
\begin{equation}\label{eq:p1}
u(\vx^{(l)},\theta^{(l,k)})=x^{(l)}_j+c
\end{equation}

\textit{2, the negation map:}
\begin{equation}\label{eq:p2}
u(\vx^{(l)},\theta^{(l,k)})=-x^{(l)}_j+c
\end{equation}

\textit{3, the maximum map:}
\begin{equation}\label{eq:p3}
\begin{aligned}
u(\vx^{(l)},\theta^{(l,k)})&=\max \{x^{(l)}_j+w,x^{(l)}_p\}+c\\ u(\vx^{(l)},\theta^{(l,k)})&=\max \{-x^{(l)}_j+w,x^{(l)}_p\}+c
\end{aligned}
\end{equation}
\end{proposition}

\begin{proof}[Proof of Proposition \ref{prop:3p}]
\textit{1, the projection map:}
Setting $u(\vx^{(l)},\theta^{(l,k)})$ as follows
\begin{equation*}
u(\vx^{(l)},\theta^{(l,k)})=\|(x^{(l)}_1,...,x^{(l)}_j+2C,...,x^{(l)}_n)\|_{\infty}-2C+c
\end{equation*}

\textit{2, the negation map:}
Setting $u(\vx^{(l)},\theta^{(l,k)})$ as follows
\begin{equation*}
u(\vx^{(l)},\theta^{(l,k)})=\|(x^{(l)}_1,...,x^{(l)}_j-2C,...,x^{(l)}_n)\|_{\infty}-2C+c
\end{equation*}

\textit{3, the maximum map:}
Setting $u(\vx^{(l)},\theta^{(l,k)})$ as follows
\begin{equation*}
\begin{aligned}
&u(\vx^{(l)},\theta^{(l,k)})\\
=&\|(x^{(l)}_1,...,x^{(l)}_j+w+2C,...,x^{(l)}_p+2C,...,x^{(l)}_n)\|_{\infty}\\
&-2C+c
\end{aligned}
\end{equation*}
\begin{equation*}
\begin{aligned}
&u(\vx^{(l)},\theta^{(l,k)})\\
=&\|(x^{(l)}_1,...,x^{(l)}_j-w-2C,...,x^{(l)}_p+2C,...,x^{(l)}_n)\|_{\infty}\\
&-2C+c
\end{aligned}
\end{equation*}
We finish the proof of Proposition \ref{prop:3p}.
\end{proof}

With three basic maps in hand, we are prepared to construct our network. Using proposition \ref{prop:3p}, the first hidden layer realizes $u(\vx,\theta^{(1,k)})=x_k$ for $k=1,...,d_{\mathrm{input}}$. The rest two units can be arbitrary, we set both to be $x_1$.

By proposition \ref{prop:3p}, throughout the whole network, we can set $u(\vx^{(l)},\theta^{(l,k)})=x_k$ for all $l$ and $k=1,...,n$. Notice that $f_i(\vx)$ can be rewritten as
\begin{equation}\label{eq:rec1}
\begin{aligned}
f_i(\vx)=-\max\{x_1-w^{(i)}_1,\max\{w^{(i)}_1-x_1,\\
\max\{...,w^{(i)}_n-x_n \}...\}\} +\tilde{g}(\vw^{(i)})
\end{aligned}
\end{equation}

Using the maximum map recurrently while keeping other units unchanged with the projection map, we can utilize the unit $u(\vx^{(l)},\theta^{(l,d_{\mathrm{input}}+1)})$ to realize one $f_i(\vx)$ at a time. Again by the use of maximum map, the last unit $u(\vx^{(l)},\theta^{(l,d_{\mathrm{input}}+2)})$ will recurrently calculate (initializing with $\max\{f_1(\vx)\}=f_1(\vx)$)
\begin{equation}\label{eq:rec2}
\max_i f_i(\vx)=\max\{f_m(\vx),\max\{...,\max\{f_1(\vx)\}...\}\}
\end{equation}

using only finite depth, say $L$, then the network outputs $g(\vx)=u(\vx^{(L)},\theta^{(L,1)})=u(\vx^{(L-1)},\theta^{(L-1,d_{\mathrm{input}}+2)})=\max_i f_i(\vx)$ as desired. We are only left with deciding a valid value for $C$. Because $\sK$ is bounded and $\tilde{g}(\vx)$ is continuous, there exists constants $C_1,C_2$ such that $\forall \vx\in \sK$, $\|\vx\|_{\infty}\le C_1$ and $|\tilde{g}(\vx)|\le C_2$, it's easy to verify that $C=2C_1+C_2$ is valid.
\end{proof}

\subsection{Proof of Theorem \ref{thm:rob}}
We prove the theorem in two steps. One step is to provide a margin bound to control the gap between standard training error and standard test error using Rademacher complexity. Next step is to bound the gap between test error and robust test error. We will first give a quick revisit on Rademacher complexity and its properties, then provide two lemmas corresponding to the two steps for the proof.

\paragraph{Rademacher Complexity}
Given a sample $X_n=\{\vx_1,...,\vx_n\}\in \sK^n$, and a real-valued function class $\sF$ on $\sK$, the Rademacher complexity of $\sF$ is defined as
$$R_n(\sF)=\mathbb{E}_{X_n}\rbr{\frac{1}{n}\mathbb{E}_{\sigma}\sbr{\sup_{f\in \sF} \sum_{i=1}^n \sigma_i f(\vx_i)}}$$    
where $\sigma_i$ are drawn from the Rademacher distribution independently, i.e. $\mathbb P(\sigma_i=1)=\mathbb P(\sigma_i=-1)=\frac{1}{2}$. It's worth noting that for any constant function $r$, $R_n(\sF)=R_n(\sF \oplus r)$ where $\sF\oplus r=\{f+r|f\in \sF\}$.

Rademacher complexity is directly related to generalization ability, as shown in Lemma \ref{lem:rade}:

\begin{lemma}\label{lem:rade}(Theorem 11 in \citet{koltchinskii2002empirical})
Let $\sF$ be a real-valued hypothesis class. For all $t>0$,
\begin{equation*}
\begin{aligned}
\mathbb P\left(\exists g\in \sF: \mathbb{E}_{(\vx, y)\sim \mathcal D}\sbr{\mathbb{I}_{y g(\vx)\le 0}}>\Phi(g)\right) \le 2e^{-2t^2}.
\end{aligned}
\end{equation*}
where
\begin{equation*}
\begin{aligned}
\Phi(g)=\inf_{\delta\in (0,1]}\left[\frac{1}{n}\sum_{i=1}^n \mathbb{I}_{y_i g(\vx_i)\le \delta}+\frac{48}{\delta}R_n(\sF)+\right.\\
\left.\rbr{\frac{\log\log_2(\frac{2}{\delta})}{n}}^{\frac{1}{2}}\right]+\frac{t}{\sqrt{n}}.
\end{aligned}
\end{equation*}
\end{lemma}
Lemma \ref{lem:rade} can be generalized to the following lemma:
\begin{lemma}\label{lem:rader}
Let $\sF$ be a real-valued hypothesis class, define the $r$-margin test error $\beta_r(g)$ as: $\beta_r(g)=\mathbb{E}_{{(\vx, y)\sim \mathcal D}}\sbr{\mathbb{I}_{y g(\vx)\le r}}$, then for any $t>0$,
$$\mathbb{P}\rbr{\exists g\in \sF: \beta_r(g)>\Phi_r(g)} \le 2e^{-2t^2}.$$
where
\begin{equation*}
\begin{aligned}
\Phi_r(g)=\inf_{\delta\in (0,1]}\left[\frac{1}{n}\sum_{i=1}^n \mathbb{I}_{y_i g(\vx_i)\le \delta+r}+\frac{48}{\delta}R_n(\sF)+\right.\\
\left.\rbr{\frac{\log\log_2(\frac{2}{\delta})}{n}}^{\frac{1}{2}}\right]+\frac{t}{\sqrt{n}}.
\end{aligned}
\end{equation*}
\end{lemma}

\begin{proof}[Proof of Lemma \ref{lem:rader}]
To further generalize Lemma \ref{lem:rade} to $\beta_r(g)$ with $r>0$, we use the fact that Rademacher complexity remain unchanged if the same constant $r$ is added to all functions in $\sF$. Lemma \ref{lem:rader} is a direct consequence by replacing $m_f$ by $m_f-r$ at the end of the proof of Theorem 11 in \cite{koltchinskii2002empirical}, where it plugs $m_f$ into Theorem 2 in \cite{koltchinskii2002empirical}.
\end{proof}

Also, it's well known (using Massart's Lemma) that Rademacher complexity can be bounded by VC dimension:

\begin{equation}\label{eq:radevc}
R_n(\sF)\le   \sqrt{\frac{2VCdim(\sF)\log\frac{en}{VCdim(\sF)}}{n}}
\end{equation}

We next bound the the gap between test error and robust test error. We have

\begin{lemma}\label{lem:robust}
Assume function $g$ is 1-Lipschitz (with respect to $\ell_\infty$-norm. The $r$-robust test error $\gamma_r(g)$ is no larger than the $r$-margin test error $\beta_r(g)$, i.e., $\gamma_r(g)\le \beta_r(g)$.
\end{lemma}

\begin{proof}[Proof of Lemma \ref{lem:robust}] Since $g(\vx)$ is 1-Lipschitz, then
$yg(\vx)> r$ implies that $\inf_{\|x'-x\|_{\infty}\le r} yg(\vx')> 0$, therefore 
$$\mathbb{E}_{\mathcal{D}}\sbr{\mathbb{I}_{y g(\vx)> r}}\le \mathbb{E}_{\mathcal{D}}\sbr{\inf_{\|x'-x\|_{\infty}\le r}\mathbb{I}_{y g(\vx')> 0}},$$ which implies $1-\beta_r\le 1-\gamma_r$.
\end{proof}

Now we are ready to prove Theorem \ref{thm:rob}. Based on Lemma \ref{lem:rader}, Lemma \ref{lem:robust} and Eqn.~\ref{eq:radevc}, it suffices to bound the VC dimension of an $\ell_\infty$-dist net. We will bound the VC dimension of an $\ell_{\infty}$-dist net by reducing it to a fully-connected ReLU network. We first introduce the VC bound for fully-connected neural networks with ReLU activation borrowed from \cite{JMLR:v20:17-612}:

\begin{lemma}\label{lem:vc1}{(Theorem 6 in \cite{JMLR:v20:17-612})}
Consider a fully-connected ReLU network architecture $F$ with input dimension $d$, width $w\ge d$ and depth (number of hidden layers) $L$, then its VC dimension satisfies:
\begin{equation}\label{eq:vc1}
VCdim(F)=\tilde{O}(L^2 w^2)   
\end{equation}
\end{lemma}

The following lemma shows how to calculate $\ell_{\infty}$-distance using a fully-connected ReLU network architecture.

\begin{lemma}\label{lem:vc}
$\forall \vw \in R^d$, there exists a fully-connected ReLU network $h$ with width $O(d)$ and depth $O(\log d)$ such that $h(\vx)=\|\vx-\vw\|_{\infty}$.
\end{lemma}

\begin{proof}[Proof of Lemma \ref{lem:vc}]
The proof is by construction. Rewrite $\|\vx-\vw\|_{\infty}$ as $\max\{x_1-w_1,w_1-x_1,...,x_d-w_d,w_d-x_d\}$ which is a maximum of $2d$ items. Notice that $\max\{x,y\}=\operatorname{ReLU}(x-y)+\operatorname{ReLU}(y)-\operatorname{ReLU}(-y)$, therefore we can use $3d$ neurons in the first hidden layer so that the input to the second hidden layer are $\max\{x_i-w_i,w_i-x_i\}$, in all $d$ items. Repeating this procedure which cuts the number of items within maximum by half, within $O(\log d)$ hidden layers this network finally outputs $\|\vx-\vw\|_{\infty}$ as desired.
\end{proof}

The VC bound of $\ell_{\infty}$-dist Net is formalized by the following lemma:

\begin{lemma}\label{lem:vc2}
Consider an $\ell_{\infty}$-dist net architecture $F$ with input dimension $d$, width $w\ge d$ and depth (number of hidden layers) $L$, then its VC dimension satisfies:
\begin{equation}\label{eq:vc2}
VCdim(F)=\tilde{O}(L^2 w^4)   
\end{equation}
\end{lemma}

\begin{proof}[Proof of Lemma \ref{lem:vc2}]
By Lemma \ref{lem:vc}, each neuron in the $\ell_{\infty}$-dist net can be replaced by a fully-connected ReLU subnetwork with width $O(w)$ and depth $O(\log w)$. Therefore a fully-connected ReLU network architecture $G$ with width $O(w^2)$ and depth $O(L \log w)$ can realize any function represented by the $\ell_{\infty}$-dist net when parameters vary. Remember that VC dimension is monotone under the ordering of set inclusion, we conclude that such $\ell_{\infty}$-dist Net architecture $F$ has VC dimension no more than that of $G$ which equals $\tilde{O}(L^2 w^4)$ by lemma \ref{lem:vc1}.
\end{proof}

Finally, Theorem \ref{thm:rob} is a direct consequence by combing Lemmas \ref{lem:rader}, \ref{lem:robust} and \ref{lem:vc2}.

\begin{remark}
Though there exist generalization bounds for general Lipschitz model class \citep{von2004distance}, the dependence on $n$ scales as $n^{-1/d}$ which suffers from the curse of dimensionality \citep{neyshabur2017exploring}. Theorem \ref{thm:rob} is dimension-free instead.
\end{remark}

\section{Interval Bound Propagation}
\label{sec_ibp}
We now give a brief description of IBP \citep{mirman2018differentiable,gowal2018effectiveness}, a simple convex relaxation method to calculate the certified radius for general neural networks. The basic idea of IBP is to compute the lower bound and upper bound for each neuron layer by layer when the input $\vx$ is perturbed. 

\paragraph{Input layer.} Let the perturbation set be an $\ell_\infty$ ball with radius $\epsilon$. Then for the input layer, calculating the bound is trivial: when $x$ is perturbed, the value in the $i$-th dimension is bounded by the interval $[x_i-\epsilon, x_i+\epsilon]$.

\paragraph{Bound propagation.} Assume the interval bound of layer $l$ has already been obtained. We denote the lower bound and upper bound of layer $l$ be $\underline{\vx}^{(l)}$ and $\overline{\vx}^{(l)}$ respectively. We mainly deal with two cases:
\begin{itemize}
\item ${\vx}^{(l)}$ is followed by a linear transformation, either a linear layer or a convolution layer. Denote $\vx^{(l+1)}=\mathbf W^{(l+1)} \vx^{(l)}+\vb^{(l+1)}$ be the linear transformation. Through some straightforward calculations, we have
\begin{equation}
\label{ibp1}
\begin{aligned}
    \underline{\vx}^{(l+1)}&=\mathbf \mu^{(l+1)}-\vr^{(l+1)}\\
    \overline{\vx}^{(l+1)}&=\mathbf \mu^{(l+1)}+\vr^{(l+1)}
\end{aligned}
\end{equation}
where
\begin{equation}
\label{ibp2}
\begin{aligned}
\mathbf \mu^{(l+1)}&=\frac 1 2 \mathbf W^{(l+1)} (\underline{\vx}^{(l)}+\overline{\vx}^{(l)})\\
\vr^{(l+1)}&=\frac 1 2 \left|\mathbf W^{(l+1)}\right| (\overline{\vx}^{(l)}-\underline{\vx}^{(l)})
\end{aligned}
\end{equation}
Here $|\cdot|$ is the element-wise absolute value operator.
\item ${\vx}^{(l)}$ is followed by a monotonic element-wise activation function, e.g. ReLU or sigmoid. Denote $\vx^{(l+1)}=\sigma(\vx^{(l)})$. Then it is straightforward to see that $\underline{\vx}^{(l+1)}=\sigma(\underline{\vx}^{(l)})$ and $\overline{\vx}^{(l+1)}=\sigma(\overline{\vx}^{(l)})$.
\end{itemize}
Using the above recurrence formulas, we can calculate the interval bound of all layers $\vx^{(l)}$.

\paragraph{Margin calculation.} Finally we need to calculate the margin vector $\vm$, the $j$th element of which is defined as the difference between neuron $x_j^{(L)}$ and $x_{y}^{(L)}$ where $y$ is the target class number, i.e. $m_j=x_j^{(L)}-x_{y}^{(L)}$. Note that $m_y=0$. It directly follows that if all elements of $\vm$ is negative (except for $m_y=0$) for any perturbed input, then we can guarantee robustness for data $\vx$. Therefore we need to get an upper bound of $\vm$, denoted as $\overline{\vm}$.

One simple way to calculate $\overline{\vm}$ is by using interval bound of the final layer to obtain $\overline m_j=\overline{x}_j^{(L)}-\underline{x}_y^{(L)}$. However, we can get a tighter bound if the final layer is a linear transformation, which is typically the case in applications. To derive a tighter bound, we first write the definition of $\vm$ into a matrix form: $$\vm=\vx^{(L)}-\mathbf 1 \mathbf e_y^T \vx^{(L)}=(\mathbf I-\mathbf 1 \mathbf e_y^T)\vx^{(L)}$$
where $\ve_y$ is a vector with all-zero elements except for the $y$th element being one, $\mathbf 1$ is the all-one vector, and $\mathbf I$ is the identity matrix. Note that $\vm$ is a linear transformation of $\vx^{(L)}$, therefore we can merge this transformation with final layer and obtain
\begin{equation}
\label{eq:ibp_margin}
\begin{aligned}
    \vm&=(\mathbf I-\mathbf 1 \mathbf e_y^T)(\mathbf W^{(L)}\vx^{(L-1)}+\vb^{(L)})\\
    &=(\mathbf W^{(L)}-\mathbf 1 \mathbf e_y^T\mathbf W^{(L)})\vx^{(L-1)}+(\vb^{(L)}-\mathbf 1 \mathbf e_y^T\vb^{(L)})\\
    &=\widetilde{\mathbf W^{(L)}}\vx^{(L-1)}+\widetilde{\vb^{(L)}}
\end{aligned}
\end{equation}
Using the same bounding technique in Eqn.~\ref{ibp1},\ref{ibp2}, we can calculate $\overline{\vm}$.

\paragraph{Loss Design.} Finally, we can optimize a loss function based on $\overline{\vm}$. We adopt the same loss function in \citet{gowal2018effectiveness,zhang2020towards} who use the combination of the natural loss and the worst perturbation loss. The loss function can be written as
\begin{align}\label{eq:cross_entropy}
l(\vg,\gT)=\frac{1}{n}\sum_{i=1}^n \kappa l_{\text{CE}}(\vg(\vx_i),y_i)+(1-\kappa) l_{\text{CE}}(\overline{\vm}_i,y_i)
\end{align}
where $l_{\text{CE}}$ denotes the cross-entropy loss, $\overline{\vm}_i$ is defined in Eqn.~\ref{eq:ibp_margin} which is calculated using convex relaxation, and $\kappa$ controls the balance between standard accuracy and robust accuracy.

As we can see, the calculation of $\overline{\vm}$ is differentiable with respect to network parameters. Therefore, any gradient-based optimizer can be used to optimize these parameters. The whole process of IBP is computationally efficient: it costs roughly two times for certification compared to normal inference. However, the bound provided by IBP is looser than other more sophisticated methods based on linear relaxation.

\section{Comparing $\ell_\infty$-dist Net with AdderNet}
\label{sec_addernet}
Recently, \citet{chen2020addernet} presents a novel form of networks called AdderNet, in which all convolutions are replaced with merely addition operations (calculating $\ell_1$-norm). Though the two networks seem to be similar at first glance, they are different indeed.

The motivation of the two networks is different. The motivation for AdderNet is to replace multiplication with additions to reduce the computational cost by using $\ell_1$-norm, while our purpose is to design robust neural networks that resist adversarial attacks by using $\ell_\infty$-norm.

The detailed implementations of the two networks are totally different. As shown in \cite{chen2020addernet}, using standard Batch Normalization is crucial to train the AdderNets successfully. However, $\ell_\infty$-dist Nets cannot adapt the standard batch normalization due to that it dramatically changes the Lipschitz constant of the network and hurt the robustness of the model. Furthermore, In AdderNet, the authors modified the back-propagation process and used a layer-wise adaptive learning rate to overcome optimization difficulties for training $\ell_1$-norm neurons. We provide a different training strategy using mean shift normalization, smoothed approximated gradients, identity-map initialization, and $\ell_p$-norm weight decay, specifically for dealing with $\ell_\infty$-norm.

Finally, the difference above leads to trained models with different properties. By using standard Batch Normalization, the AdderNet can be trained easily but is not robust even with respect to $\ell_1$-norm perturbation. In fact, even without Batch Normalization, we still cannot provide robust guarantee for AdderNet. Remark \ref{remark_lip} has shown that the Lipschitz property holds specifically for $\ell_\infty$-dist neurons rather than other $\ell_p$-dist neurons.


\section{Experimental Details}
\label{sec_experiment_details}
We use a single NVIDIA RTX-3090 GPU to run all these experiments. Each result of $\ell_\infty$-dist net in Table \ref{tbl_results} is reported using the medien of 8 runs with the same hyper-parameters. Details of network architectures are provided in Table \ref{tbl_architecture}. Details of hyper-parameters are provided in Table \ref{tbl_hyperparameter}. For $\ell_\infty$-dist Net, we use multi-class hinge loss with a threshold hyper-parameter $t$ which depends on the robustness level $\epsilon$. For $\ell_\infty$-dist Net+MLP, we use IBP loss with two hyper-parameters $\kappa$ and $\epsilon_{\text{train}}$ as in \citet{gowal2018effectiveness,zhang2020towards}. We use a linear warmup over $\epsilon_{\text{train}}$ in the first $e_1$ epoch while keeping $p=p_{\text{start}}$ fixed, increase $p$ from $p_{\text{start}}$ to $p_{\text{end}}$ in the next $e_2$ epochs while keeping $\epsilon_{\text{train}}$ fixed, and fix $p=\infty$ in the last $e_3$ epochs. Different from \citet{gowal2018effectiveness}, $\kappa$ is kept fixed throughout training since we do not find any training instability with the fixed $\kappa$.

For other methods, the results are typically borrowed from the original paper. For IBP and CROWN-IBP results on Fashion-MNIST dataset, we use the official github repo of CROWN-IBP and perform a grid search over hyper-parameters $\kappa$ and set $\epsilon_{\text{train}}=1.1\epsilon_{\text{test}}=0.11$. We use the largest model in their papers (denoted as DM-large) and use the learning rate and epoch schedule the same as in the MNIST dataset. We select the hyper-parameter $\kappa$ with the best certified accuracy.

In Table \ref{tab:speed}, the baseline results are run using the corresponding official github repos. For example, we measure the speed for IBP and CROWN-IBP using the github repo of \citet{zhang2020towards}, and measure the speed for CROWN-IBP with loss fusion using the github repo of \citet{xu2020automatic}.

\section{Experiments for Identity-map Initialization}
\label{sec_initialization}

We conduct experiments to see the the problem of Gaussian initialization for training deep models. The results are shown in Table \ref{tbl_identity}, where we train $\ell_\infty$-dist Nets with different number of layers on CIFAR-10 dataset, using the same hyper-parameters provided in Table \ref{tbl_hyperparameter}. It is clear from the table that the \emph{training accuracy} begins to drop when the model goes deeper. After applying identity-map initialization, the training accuracy does not drop for deep models.

\begin{table*}[]
\caption{Details of network architectures. Here ``Norm'' denotes mean shift normalization in Section \ref{sec_bn}.}
\small
\vspace{2pt}
\label{tbl_architecture}
\begin{center}
\begin{tabular}{ccccc}
\hline
       & $\ell_\infty$-dist Net (MNIST)         & $\ell_\infty$-dist Net+MLP (MNIST)       & $\ell_\infty$-dist Net (CIFAR-10)      & $\ell_\infty$-dist Net+MLP (CIFAR-10)    \\ \hline
Layer1 & $\ell_\infty$-dist(5120)+Norm & $\ell_\infty$-dist(5120)+Norm & $\ell_\infty$-dist(5120)+Norm & $\ell_\infty$-dist(5120)+Norm \\
Layer2 & $\ell_\infty$-dist(5120)+Norm & $\ell_\infty$-dist(5120)+Norm & $\ell_\infty$-dist(5120)+Norm & $\ell_\infty$-dist(5120)+Norm \\
Layer3 & $\ell_\infty$-dist(5120)+Norm & $\ell_\infty$-dist(5120)+Norm & $\ell_\infty$-dist(5120)+Norm & $\ell_\infty$-dist(5120)+Norm \\
Layer4 & $\ell_\infty$-dist(5120)+Norm & $\ell_\infty$-dist(5120)+Norm & $\ell_\infty$-dist(5120)+Norm & $\ell_\infty$-dist(5120)+Norm \\
Layer5 & $\ell_\infty$-dist(10)                 & FC(512)+Tanh                        & $\ell_\infty$-dist(5120)+Norm & $\ell_\infty$-dist(5120)+Norm \\
Layer6 &                                        & FC(10)                                 & $\ell_\infty$-dist(10)                 & FC(512)+Tanh                        \\
Layer7 &                                        &                                        &                                        & FC(10)                                 \\ \hline
\end{tabular}
\end{center}
\vspace{-10pt}
\end{table*}

\begin{table*}[]
\caption{Hyper-parameters to produce results in Table \ref{tbl_results}.}
\small
\vspace{2pt}
\label{tbl_hyperparameter}
\begin{center}
\begin{tabular}{c|ccccccc}
\hline
Dataset            & \multicolumn{2}{c|}{MNIST}                                             & \multicolumn{2}{c|}{FashionMNIST}                                      & \multicolumn{2}{c|}{CIFAR-10}   & TinyImageNet                   \\ \hline
Architecture       & $\ell_\infty$ Net & \multicolumn{1}{c|}{$\ell_\infty$ Net+MLP} & $\ell_\infty$ Net & \multicolumn{1}{c|}{$\ell_\infty$ Net+MLP} & $\ell_\infty$ Net & \multicolumn{1}{c|}{$\ell_\infty$ Net+MLP} & $\ell_\infty$ Net+MLP\\ \hline
Optimizer        & \multicolumn{7}{c}{Adam($\beta_1=0.9,\beta_2=0.99,\epsilon=10^{-10}$)}                                            \\
Batch size         & \multicolumn{7}{c}{512}                                                                                                                                                                             \\
Learning rate      & \multicolumn{7}{c}{0.02 (0.04 for TinyImageNet)}                                                                                                                                                                            \\
Weight decay       & \multicolumn{7}{c}{0.005}                                                                                                                                                                           \\
$p_{\text{start}}$        & \multicolumn{7}{c}{8}                                                                                                                                                                               \\
$p_{\text{end}}$          & \multicolumn{7}{c}{1000}                                                                                                                                                                            \\ \hline
$\epsilon_{\text{train}}$ & -                      & \multicolumn{1}{c|}{0.35}                     & -                      & \multicolumn{1}{c|}{0.11}                     & -                      & \multicolumn{1}{c|}{8.8/255}                  & 0.005\\
$\kappa$           & -                      & \multicolumn{1}{c|}{0.5}                     & -                      & \multicolumn{1}{c|}{0.5}                     & -                      & \multicolumn{1}{c|}{0} &0                        \\
$t$                & 0.9                    & \multicolumn{1}{c|}{-}                        & 0.45                   & \multicolumn{1}{c|}{-}                        & 80/255                 & \multicolumn{1}{c|}{-}                      &-  \\
$e_1$              & 50 & \multicolumn{1}{c|}{50}                                                & 50 & \multicolumn{1}{c|}{50}                                                & 100 & \multicolumn{1}{c|}{100}        &100                   \\
$e_2$              & 300 & \multicolumn{1}{c|}{300}                                               & 300 & \multicolumn{1}{c|}{300}                                               & 650 & \multicolumn{1}{c|}{650}  &350                         \\
$e_3$              & 50 & \multicolumn{1}{c|}{50}                                                & 50 & \multicolumn{1}{c|}{50}                                                & 50 & \multicolumn{1}{c|}{50}    &50                       \\
Total epochs       & 400 & \multicolumn{1}{c|}{400}                                               & 400 & \multicolumn{1}{c|}{400}                                               & 800 & \multicolumn{1}{c|}{800}  &500                         \\ \hline
\end{tabular}
\end{center}
\end{table*}

\begin{table*}[]
\caption{Performance of $\ell_\infty$-dist Net trained using different initialization methods on CIFAR-10 dataset.}
\small
\vspace{2pt}
\label{tbl_identity}
\begin{center}
\begin{tabular}{c|cc|cc}
\hline
\multirow{2}{*}{\#Layers} & \multicolumn{2}{c|}{Gaussian Initialization} & \multicolumn{2}{c}{Identity Map Initialization} \\ \cline{2-5} 
                          & Train                 & Test                 & Train                  & Test                   \\ \hline
5                         & 63.01                 & 55.76                & 65.93                  & 56.56                  \\
6                         & 63.46                 & 55.85                & 65.51                  & 56.80                  \\
7                         & 63.82                 & 56.77                & 66.58                  & 57.06                  \\
8                         & 63.46                 & 56.72                & 67.11                  & 56.64                  \\
9                         & 61.57                 & 55.94                & 67.52                  & 57.01                  \\
10                        & 58.72                 & 55.04                & 68.84                  & 57.18                  \\ \hline
\end{tabular}
\end{center}
\end{table*}

\section{Preliminary Results for Convolutional $\ell_\infty$-dist Nets}
\label{sec_conv}
$\ell_\infty$-dist neuron can be easily used in convolutional neural networks, and the Lipschitz property still holds. We also conduct experiments using convolutional $\ell_\infty$-dist net on the CIFAR-10 dataset. We train a a eight-layer convolutional $\ell_\infty$-dist Net+MLP. The detailed architecture is given in Table \ref{tbl_conv_architecture}. The training configurations and hyper-parameters are exactly the same as training fully-connected $\ell_\infty$-dist nets. Results are shown in Table \ref{tbl_results_conv}. From Table \ref{tbl_results_conv}, we can see that convolutional $\ell_\infty$-dist nets still reach good certified accuracy and outperforms all existing methods on the CIFAR-10 dataset.


\begin{table*}[]
\caption{Details of convolutional network architectures. Here ``Norm'' denotes mean shift normalization in Section \ref{sec_bn}.}
\small
\vspace{5pt}
\label{tbl_conv_architecture}
\begin{center}

\begin{tabular}{cc}
\hline
             & Convolutional $\ell_\infty$-dist Net+MLP (CIFAR-10)    \\ \hline
Layer1 &  $\ell_\infty$-dist-conv(3, 128, kernel=3)+Norm \\
Layer2 &  $\ell_\infty$-dist-conv(128, 128, kernel=3)+Norm \\
Layer3 & $\ell_\infty$-dist-conv(128, 256, kernel=3, stride=2)+Norm \\
Layer4 & $\ell_\infty$-dist-conv(256, 256, kernel=3)+Norm \\
Layer5 &  $\ell_\infty$-dist-conv(256, 256, kernel=3)+Norm \\
Layer6 &  $\ell_\infty$-dist(512)+Norm \\
Layer7 &  FC(512)+Tanh \\
Layer8 & FC(10) \\
\hline
\end{tabular}
\end{center}
\end{table*}

\begin{table*}[]
\caption{Performance of convolutional $\ell_\infty$-dist nets and existing methods.}
\small
\vspace{2pt}
\label{tbl_results_conv}
\begin{center}
\begin{tabular}{c|cr|rrr}
\hline
Dataset                                                                                     & Method           & FLOPs & Test          & Robust        & Certified \\ \hline
\multirow{6}{*}{\begin{tabular}[c]{@{}c@{}}CIFAR-10\\ ($\epsilon=8/255$)\end{tabular}}      & PVT                        &2.4M                      & 48.64                             & 32.72                             & 26.67                                  \\
                                                                                            & DiffAI                      &96.3M                     & 40.2                              & -                                 & 23.2                                                                            \\
                                                                                            & COLT                       &6.9M                      & \textbf{51.7}                              & -                                 & 27.5                                                      \\
                                                                                            & IBP                        &151M                      & 50.99                             & 31.27                             & 29.19                                    \\
                                                                                            & CROWN-IBP                  &151M                      & 45.98                             & 34.58                             & 33.06                                                                             \\
                                                                                            &  Conv $\ell_\infty$-dist Net+MLP                      &566M                                    & 49.17                             & \textbf{37.23}                             & \textbf{34.30}                                         \\
                                                                                             \hline
\end{tabular}
\end{center}
\vspace{-15pt}
\end{table*}

\end{document}